\documentclass{colt2015} % Anonymized submission

\DeclareMathAlphabet{\pazocal}{OMS}{zplm}{m}{n}
\usepackage{mathrsfs,bbm,euscript,dsfont}
\usepackage{wasysym,enumerate}
\usepackage{emptypage}
\def\rn{\mathbb{R}}
\def\cn{\mathbb{C}}

\def\l{\left}
\def\r{\right}

\def\sF{\pazocal{F}}
\def\sG{\pazocal{G}}
\def\sM{\pazocal{M}}
\def\sX{\pazocal{X}}
\def\sD{\pazocal{D}}

\def\sV{\pazocal{V}}

\def\sP{\mathscr{P}}
\def\sQ{\mathscr{Q}}
\def\bX{\mathbf{X}}

\def\hs{\mathscr{HS}}
\def\pr{\mathbb{P}}

\def\d2{\sD_2}
\def\dd{\Delta\left( \sD \right)}

\def\span{\operatorname{span}}
\def\simiid{\overset{iid}{\sim}}
\allowdisplaybreaks
\title[Identifiability of Latent Variables from Sample Groups]{On The Identifiability of Mixture Models from Grouped Samples}
\usepackage{times}
 \coltauthor{\Name{Robert A. Vandermeulen}   \Email{rvdm@umich.edu}\\
\Name{Clayton D. Scott} \Email{clayscot@umich.edu}\\
\addr Department of Electrical Engineering and Computer Science\\
\addr University of Michigan\\
Ann Arbor, MI 48109
}
\begin{document}

\maketitle

\begin{abstract}
Finite mixture models are statistical models which appear in many problems in statistics and machine learning. In such models it is assumed that data are drawn from random probability measures, called mixture components, which are themselves drawn from a probability measure $\sP$  over probability measures. When estimating mixture models, it is common to make assumptions on the mixture components, such as parametric assumptions. In this paper, we make no assumption on the mixture components, and instead assume that observations from the mixture model are grouped, such that observations in the same group are known to be drawn from the same component. We show that any mixture of $m$ probability measures can be uniquely identified provided there are $2m-1$ observations per group. Moreover we show that, for any $m$, there exists a mixture of $m$ probability measures that cannot be uniquely identified when groups have $2m-2$ observations. Our results hold for any sample space with more than one element.
\end{abstract}


\section{Introduction}

A finite mixture model is a probability law based on a finite number of probability measures, $\mu_1, \ldots, \mu_m$, and a discrete distribution $w_1, \ldots, w_m$. A realization of a mixture model is first generated by first generating a component at random $k$, $1 \le k \le m$, and then drawing from $\mu_k$. A mixture model can be associated with a probability measure on probability measures, which we denote $\sP$. Mixture models are used to model data throughout statistics and machine learning.

A primary theoretical question concerning mixture models is identifiability. A mixture model is said to be identifiable if no other mixture model (of equal or lesser complexity) explains the distribution of the data. Some previous work on identifiability considers the situation where the observations are drawn iid from the mixture model, and conditions on $\mu_1, \ldots, \mu_m$ are imposed, such as Gaussianity \citep{dasgupta07,anderson14}. In this work we make no assumptions on $\mu_1,\ldots,\mu_m$. Instead, we assume the observations are grouped, such that realizations from the same group are known to be iid from the same component. We call these groups of samples ``random groups.'' We define a random group to be a random collection $\bX_i$, where $\bX_i =X_{i,1},\dots,X_{i,n} \simiid \mu_i$ and $\mu_i \simiid \sP$.

Consider the set of all mixtures of probability measures which yield the same distribution over the random groups as does $\sP$. If some element of this set other than $\sP$ has no more components than $\sP$ then $\sP$ is not identifiable. In other words, there is no way to differentiate $\sP$ from another model of equal or lesser complexity. Fortunately, with a sufficient number of samples in each random group, $\sP$ becomes the most simple model which describes the data. In this paper we show that, for any sample space, any mixture of probability measures with $m$ components is identifiable when there are $2m-1$ samples per random group. Furthermore we show that this bound cannot be improved, regardless of sample space.

\subsection{Applications of Probability Measures over Probability Measures}

Though a somewhat mathematically abstract object, probability measures over spaces of probability measures arise quite naturally in many statistical problems. Any application which use mixture models, for example clustering, is utilizing a probability measure over probability measures. Moreover mixture models are a subset of a larger class of models known as latent variable models. One problem in latent variable models which has seen significant interest recently is topic modeling. Topic modelling is concerned with the extraction of some sort of topical structure from a collection of documents. Many popular methods for topic modelling assume that each document in question has a latent variable representing a ``topic'' or a random convex combination of topics which determines the distribution of words in that document \citep{LDA,anandkumar14,arora12}.

Another statistical problem which often utilizes a probability measure over probability measures is transfer learning. In transfer learning one is interested in utilizing several different but related training datasets (perhaps a collection of datasets which correspond to different patients in a study) to construct some sort of classifier or regressor for another different but related testing dataset. There are many approaches to this problem but one formulation assumes that each dataset is generated from a random probability measure and each random measure is generated from a fixed probability measure over probability measures \citep{blanchard11,maurer13}.

Finally sometimes we would like to perform statistical techniques directly on a space of probability measures. Examples of this include detection of anomalous distributions \citep{muandet13} and distribution regression \citep{poczos13,szabo14}. 

\subsection{How Does Group Size Affect Consistency?}
Many of the applications above assume a model similar to the one we described in the first paragraph. They assume there exists some probability measure, $\sP$, over a space of probability measures  from which we have observed groups of data $\bX_1,\ldots, \bX_N$ with $\bX_i =  X_{i,1},\ldots,X_{i,M_i} \simiid p_i $ and $p_i \sim \sP$. For example in topic modeling each $\bX_i$ is a document which contains $M_i$ words and in transfer learning $\bX_i$ is one of the several different training datasets. Proposed algorithms for solving these problems often contain some sort of consistency result and these results typically require that $N\to\infty$ and either $M_i \to \infty$ for all $i$ or that $\sP$ satisfies some properties which makes $M_i \to \infty$ unnecessary. When considering such results one may wonder what sort of statistical penalty we incur from fixing $M_i=C$  for all $i$.

While this question is clearly interesting from a theoretical perspective it has a couple of important practical implications. Firstly it is not uncommon for $C$ to be restricted in practice. An example of this is topic modelling of Twitter documents, where the restricted character count keeps each $M_i$ quite small. The second important practical consideration is that some latent variable techniques do not utilize the full sample $\bX_i$ and instead break down $\bX_i$ into many pairs or triplets of samples for analysis \citep{anandkumar14,arora12}. It is important to know what, if anything, is lost from doing this. Though we do not provide a direct answer to this question, our results seem to suggest that such techniques may significantly limit what can be known about $\sP$.

\section{Related Work}
The question of how many samples are necessary in each random group to uniquely identify a finite mixture of measures has come up sporadically over the past couple of decades. The application of Kruskal's theorem \citep{kruskal77} has been used to concoct various identifiability results for random groups containing three samples. In \cite{allman09} it was shown that any mixture of linearly independent measures over a discrete space or linearly independent probability distributions on $\rn^d$ are identifiable from random groups containing three samples. In \cite{hett00} it was shown that a mixture of $m$ probability measures on $\rn$ is identifiable from random groups of size $2m-1$ provided there exists some point in $\rn$ where the cdf of each mixture component at that point is distinct. The result most closely resembling our own is in \cite{rabani13}. In that paper they show that a mixture of $m$ probability measures over a discrete domain is identifiable with $2m-1$ samples in each random group. They also show that this bound is tight and provide a consistent algorithm for estimating arbitrary mixtures of measures over a discrete domain.

Our proofs are quite different from other related identifiability results and rely on tools from functional analysis. Other results in the same vein as ours rely on algebraic or spectral theoretic tools. Our proofs basically rely on two proof techniques. The first technique is the embedding of finite collections of measures in some Hilbert space. The second technique is using the properties of symmetric tensors over $\rn^d$ and applying them to tensor products of Hilbert spaces. Our proofs are not totally detached from the algebraic techniques but the algebraic portions are hidden away in previous results about symmetric tensors.
\section{Problem Setup}
We will be treating this problem in as general of a setting as possible. For any measurable space we define $\delta_x$ as the Dirac measure at $x$. For $\smiley$ a set, $\sigma$-algebra, or measure, we denote $\smiley^{\times a}$ to be the standard $a$-fold product associated with that object. For any natural number $k$ we define $\left[ k \right] \triangleq \mathbb{N} \bigcap \left[ 1,k \right]$.
Let $\Omega$ be a set containing more than one element. This set is the sample space of our data. Let $\sF$ be a $\sigma$-algebra over $\Omega$. Assume $\sF \neq \left\{ \emptyset, \Omega \right\}$. We denote the space of probability measures over this space as $\sD\left( \Omega,\sF \right)$, which we will shorten to $\sD$. We will equip $\sD$ with the $\sigma$-algebra $2^\sD$ so that each Dirac measure over $\sD$ is unique. Define $\dd \triangleq \span\left( \delta_x: x \in \sD \right)$. This will be the ambient space where our mixtures of probability measures live. Let $\sP = \sum_{i=1}^m \delta_{\mu_i} w_i $ be a probability measure in $\dd$. Let $\mu\sim \sP$ and $X_1 ,\ldots, X_n \simiid \mu$. We will denote $\bX = \left( X_1,\ldots,X_n \right)$.

We will now derive the probability law of $\bX$. Let $A\in \Omega^{\times n}$, we have
\begin{eqnarray*}
	\pr\left(\bX \in A \right)
	&=& \sum_{i=1}^m \pr\left( \bX \in A \right|\mu=\mu_i) \pr\left( \mu=\mu_i \right)\\
	&=&  \sum_{i=1}^m w_i \mu_i^{\times n}\left( A \right).
\end{eqnarray*}
The second equality follows from Lemma 3.10 in \cite{fomp}.
So the probability law of $\bX$ is 
\begin{eqnarray}
	\label{xdens}
	\sum_{i=1}^m w_i \mu_i^{\times n}. 
\end{eqnarray}
We want to view the probability law of $\bX$ as a function of $\sP$ in a mathematically rigorous way, which requires a bit of technical buildup.
Let $\sV$ be a vector space. We will now construct a version of the integral for $\sV$-valued functions over $\sD$. Let $\sQ\in \dd$. From the definition of $\dd$ it follows that $\sQ$ admits the representation $$\sQ = \sum_{i=1}^r \delta_{\mu_i} \alpha_i.$$
From the well-ordering principle there must exist some representation with minimal $r$ and we define $r$ as the {\it order} of $\sQ$. We can show that the representation of any $\sQ \in \dd$ is unique up to permutation of its indices.
\begin{definition} \label{def:mixmeasure}
	We call $\sP$ a {\em mixture of measures} if it is a probability measure in $\dd$. We will say that $\sP$ has $m$ {\em mixture components} if it has order $m$.
\end{definition}

\begin{lemma} \label{lem:represent}
	Let $\sQ\in \dd$ and admit minimal representations $\sQ = \sum_{i=1}^r \delta_{\mu_i} \alpha_i = \sum_{i=1}^r\delta_{\mu_i'} \alpha_i'$. There exists some permutation $\psi:\left[ r \right] \to \left[ r \right]$ such that $\mu_{\psi\left( i \right)} = \mu'_i$ and $\alpha_{\psi\left( i \right)} = \alpha'_i$ for all $i$.
\end{lemma}
\begin{proof}
	Because both representations are minimal it follows that $\alpha'_i \neq 0$ for all $i$ and $\mu_i' \neq \mu_j'$ for all $i \neq j$. From this we know $\sQ\left( \l\{\mu_i'\r\} \right) \neq 0$ for all $i$. Because $\sQ\left( \l\{\mu_i'\r\} \right) \neq 0$ for all $i$ it follows that for any $i$ there exists some $j$ such that $\mu_i' = \mu_j$. Let $\psi: \left[ r \right] \to \left[ r \right]$ be a function satisfying $\mu_i' = \mu_{\psi\left( i \right)}$. Because the elements $\mu_1,\ldots,\mu_r$ are also distinct $\psi$ must be injective and thus a permutation. Again from this distinctness we get that, for all $i$, $\sQ\left( \left\{ \mu_i' \right\}  \right)= \alpha'_i =\alpha_{\psi\left( i \right)}$ and we are done.
\end{proof}
Henceforth when we define an element of $\dd$ with a summation we will assume that the summation is a minimal representation. Any minimal representation of a mixture of measures $\sP$ with $m$ components satisfies $\sP=\sum_{i=1}^m w_i \delta_{\mu_i}$ with $w_i>0$ for all $i$ and $\sum_{i=1}^m w_i = 1$. So any mixture of measures is a convex combination of Dirac measures at elements in $\sD$.

 For a function $f:\sD \to \sV$ define
\begin{eqnarray*}
	\int f(\mu) d\sQ(\mu) = \sum_{i=1}^r \alpha_i f\left( \mu_i \right),
\end{eqnarray*}
where $\sum_{i=1}^r \delta_{\mu_i} \alpha_i$ is a minimal representation of $\sQ$. This integral is well defined as a consequence of Lemma \ref{lem:represent}.

For a $\sigma$-algebra $\left(Q, \Sigma \right)$ we define $\sM \left(Q, \Sigma \right)$ as the space of all finite signed measures over that space. Let $\lambda_n:\sM\left( \Omega, \sF \right) \to \sM\left( \Omega^{\times n}, \sF^{\times n}\right);\mu \mapsto \mu^{\times n}$. We introduce the operator $V_n:\dd\to \sM\left( \Omega^{\times n}, \sF^{\times n} \right)$
\begin{eqnarray*}
	V_n(\sQ) = \int \lambda_n(\mu) d\sQ\left( \mu \right)= \int \mu^{\times n} d\sQ\left( \mu \right).
\end{eqnarray*}
For a minimal representation $\sQ =\sum_{i=1}^r \delta_{\mu_i} \alpha_i$, we have
\begin{eqnarray*}
	V_n(\sQ) =\sum_{i=1}^r \mu_i^{\times n} \alpha_i.
\end{eqnarray*}

From this definition we have that $V_n\left( \sP \right)$ is simply the law of $\bX$ which we derived earlier.
Two mixtures of measures are different if they admit a different measure over $\sD$.
\begin{definition}\label{def:ident}
	We call a mixture of measures, $\sP$, \emph{$n$-identifiable} if there does not exist a different mixture of measures $\sP'$, with order no greater than the order of $\sP$, such that $V_n\left( \sP \right) = V_n\left( \sP' \right)$.
\end{definition}

Definition \ref{def:ident} is the central object of interest in this paper. Given a mixture of measures, $\sP = \sum_{i=1}^m w_i\delta_{\mu_i}$ then $V_n(\sP)$ is equal to $\sum_{i=1}^m w_i \mu_i^{\times n}$, the measure from which $\bX$ is drawn. In topic modelling $\bX$ would be the samples from a single document and in transfer learning it would  be one of the several collections of training samples. If $\sP$ is not $n$-identifiable then we know that there exists a mixture of measures which is no more complex (in terms of number of mixture components) than $\sP$ which is not discernible from $\sP$ given the data. Practically speaking this means we need more samples in each random group $\bX$ in order for the full richness of $\sP$ to be manifested in $\bX$.
\section{Results}
Our primary result gives us a bound on the $n$-identifiability of all mixtures of measures with $m$ or fewer components. We also show that this bound is tight.
\begin{theorem} \label{thm:ident}
	Let $\left( \Omega,\sF \right)$ be a measurable space. Mixtures of measures with $m$ components are $(2m-1)$-identifiable.
\end{theorem}

\begin{theorem} \label{thm:noident}
	Let $\left( \Omega, \sF \right)$ be a measurable space with $\sF \neq \left\{ \emptyset,\Omega \right\}$. For all $m$, there exists a mixture of measures with $m$ components which is not $(2m-2)$-identifiable.
\end{theorem}
Unsurprisingly, if a mixture of measures is $n$-identifiable then it is $q$-identifiable for all $q>n$. Likewise if a mixture of measures is not $n$-identifiable then it is not $q$-identifiable for $q<n$. Thus identifiability is, in some sense, monotonic. 
\begin{lemma}
	If a mixture of measures is $n$-identifiable then it is $q$-identifiable for all $q>n$.
\end{lemma}
\begin{proof}
	We will proceed by contradiction. Let $\sP = \sum_{i=1}^l a_i \delta_{\mu_i}$ be $n$-identifiable, let $\sP' = \sum_{j=1}^r b_j \delta_{\nu_j}$ be a different mixture of measures with $r\le l$ and 
\begin{eqnarray*}
	\sum_{i=1}^l a_i \mu_i^{\times q} = \sum_{j=1}^r b_j \nu_j^{\times q}
\end{eqnarray*}
for some $q>n$. Let $A \in \sF^{\times n}$ be arbitrary. We have
\begin{eqnarray*}
	\sum_{i=1}^l a_i \mu_i^{\times q} &=& \sum_{j=1}^r b_j \nu_j^{\times q}\\
	\Rightarrow \sum_{i=1}^l a_i \mu_i^{\times q}\left( A\times \Omega^{\times q-n} \right) &=& \sum_{j=1}^r b_j \nu_j^{\times q}\left( A\times \Omega^{\times q-n} \right)\\
	\Rightarrow \sum_{i=1}^l a_i \mu_i^{\times n}\left( A \right) &=& \sum_{j=1}^r b_j \nu_j^{\times n}\left( A  \right).
\end{eqnarray*}
This implies that $\sP$ is not $n$-identifiable, a contradiction.
\end{proof}
\begin{lemma} \label{lem:noident}
	If a mixture of measures is not $n$-identifiable then it is not $q$-identifiable for any $q<n$.
\end{lemma}
\begin{proof}
	Let a mixture of measures $\sP = \sum_{i=1}^l a_i \delta_{\mu_i}$ not be $n$-identifiable. It follows that there exists a different mixture of measures $\sP' = \sum_{j=1}^r b_j \delta_{\nu_j}$, with $r\le l$, such that
\begin{eqnarray*}
	\sum_{i=1}^l a_i \mu_i^{\times n} &=& \sum_{j=1}^r b_j \nu_j^{\times n}.
\end{eqnarray*}
Let $A \in \sF^{\times q}$ be arbitrary, we have
\begin{eqnarray*}
	\sum_{i=1}^l a_i \mu_i^{\times n}\left( A\times \Omega^{\times n-q} \right) &=& \sum_{j=1}^r b_j \nu_j^{\times n}\left( A\times \Omega^{\times n-q} \right)\\
	\Rightarrow \sum_{i=1}^l a_i \mu_i^{\times q}\left( A  \right) &=& \sum_{j=1}^r b_j \nu_j^{\times q}\left( A \right)
\end{eqnarray*}
and therefore $\sP$ is not $q$-identifiable.
\end{proof}
Viewed alternatively these results say that $n=2m-1$ is the smallest value for which $V_{n}$ is injective over the set of all minimal mixtures of measures with $m$ or fewer components.
\section{Tensor Products of Hilbert Spaces}
Our proofs will rely heavily on the geometry of tensor products of Hilbert spaces which we will introduce in this section.

\subsection{Overview of Tensor Products}
First we introduce tensor products of Hilbert spaces. To our knowledge there does not exist a rigorous construction of the tensor product Hilbert space which is both succinct and intuitive. Because of this we will simply state some basic facts about tensor products of Hilbert spaces and hopefully instill some intuition for the uninitiated by way of example. A through treatment of tensor products of Hilbert spaces can be found in \cite{kadison83}.

Let $H$ and $H'$ be Hilbert spaces. From these two Hilbert spaces the ``simple tensors'' are elements of the form $h\otimes h'$ with $h\in H$ and $h' \in H'$. We can treat the simple tensors as being the basis for some inner product space $H_0$, with the inner product of simple tensors satisfying
\begin{eqnarray*}
	\l<h_1 \otimes h_1', h_2 \otimes h_2'\r> = \l<h_1,h_2\r>\l<h_1',h_2'\r>.
\end{eqnarray*}
The tensor product of $H$ and $H'$ is the completion of $H_0$ and is denoted $H\otimes H'$. To avoid potential confusion we note that notation just described is standard in operator theory literature. In some literature our definition of $H_0$ is denoted as $H\otimes H'$ and our definition of $H \otimes H'$ is denoted $H \widehat{\otimes} H'$.

As an illustrative example we consider the tensor product $L^2\left( \rn \right) \otimes L^2\left( \rn \right)$. It can be shown that there exists an isomorphism between $L^2\left( \rn \right) \otimes L^2\left( \rn \right)$ and $L^2(\rn^2)$ which maps the simple tensors to separable functions, $f \otimes f' \mapsto f(\cdot)f'(\cdot)$. We can demonstrate this isomorphism with a simple example. Let $f,g,f',g'\in L^2\left( \rn \right)$. Taking the $L^2(\rn^2)$ inner product of $f(\cdot)f'(\cdot)$ and $g(\cdot)g'(\cdot)$ gives us 

\begin{eqnarray*}
\int\int \l(f(x)f'(y)\r)\l(g(x)g'(y\r)) dx dy 
&=& \int f(x)g(x) dx \int f'(y)g'(y) dy\\
&=& \l<f,g\r>  \l<f',g'\r>\\
&=& \l<f\otimes f', g \otimes g'\r>.
\end{eqnarray*}

Beyond tensor product we will need to define tensor power. To begin we will first show that tensor products are, in some sense, associative. Let $H_1, H_2, H_3$ be Hilbert spaces. Proposition 2.6.5 in \cite{kadison83} states that there is a unique unitary operator, $U: (H_1 \otimes H_2)\otimes H_3 \to H_1 \otimes (H_2 \otimes H_3)$, which satisfies the following for all $h_1 \in H_1, h_2 \in H_2, h_3 \in H_3$,
\begin{eqnarray*}
	U\left( \left( h_1 \otimes h_2 \right)\otimes h_3 \right) = h_1 \otimes \left( h_2 \otimes h_3 \right).
\end{eqnarray*}
This implies that for any collection of Hilbert spaces, $H_1,\ldots , H_n$, the Hilbert space $H_1 \otimes \cdots \otimes H_n$ is defined unambiguously regardless of how we decide to associate the products. In the space $H_1 \otimes \cdots \otimes H_n$ we define a simple tensor as a vector of the form $h_1 \otimes\cdots\otimes h_n$ with $h_i \in H_i$. In \cite{kadison83} it is shown that $H_1 \otimes\cdots \otimes H_n$ is the closure of the span of these simple tensors. To conclude this primer on tensor products we introduce the following notation. For a Hilbert space $H$ we denote $H^{\otimes n}= \underbrace{H\otimes H \otimes \dots \otimes H}_\text{n times}$ and for $h \in H$, $h^{\otimes n}= \underbrace{h\otimes h \otimes \dots \otimes h}_\text{n times}$.

\subsection{Some Results for Tensor Product Spaces}
We will derive state technical results which will be useful for the rest of the paper. These lemmas are similar to or  are straightforward extensions of previous results which we needed to modify for our particular purposes. Let $\left( \Psi, \sG, \mu \right)$ be a $\sigma$-finite measure space. We have the following lemma which connects the $L^2$ space of products of measures to the tensor products of the $L^2$ space for each measure. The proof of this lemma is straightforward but technical and can be found in the appendix. 
\begin{lemma}
	\label{lem:l2prod}
	There exists a unitary transform $U:L^2\left( \Psi, \sG, \mu \right)^{\otimes n} \to L^2\left( \Psi^{\times n}, \sG^{\times n}, \mu^{\times n} \right)$ such that, for all $f_1,\ldots, f_n \in L^2\left( \Psi, \sG, \mu \right)$, $U\left( f_1\otimes \cdots \otimes f_n \right) = f_1(\cdot)\cdots f_n(\cdot)$.
\end{lemma}
The following lemma used in the proof of Lemma \ref{lem:l2prod} as well as the proof of Theorem \ref{thm:noident}. The proof of this lemma is also not particularly interesting and can be found in the appendix.
\begin{lemma} \label{lem:unitprod}
	Let $H_1,\ldots, H_n, H_1',\ldots, H_n'$ be a collection of Hilbert spaces and $U_1,\ldots,U_n$ a collection of unitary operators with $U_i:H_i \to H_i'$ for all $i$. There exists a unitary operator $U:H_1 \otimes \cdots \otimes H_n \to H_1' \otimes \cdots \otimes H_n'$ satisfying $U\left( h_1 \otimes\cdots \otimes h_n \right) = U_1(h_1) \otimes \cdots \otimes U_n(h_n)$ for all $h_1 \in H_1 ,\ldots, h_n \in H_n$.
\end{lemma}
\begin{lemma}\label{lem:linind}
	Let $n>1$ and let $h_1,\ldots, h_n$ be elements of a Hilbert space such that no elements are zero and no pairs of elements are collinear. Then $h_1^{\otimes n-1},\ldots ,h_n^{\otimes n-1}$ are linearly independent.
\end{lemma}
A statement of this lemma for $\rn^d$ can be found in \cite{symtensorrank}. We present our own proof for the Hilbert space setting.
\begin{proof}
	We will proceed by induction. For $n=2$ the lemma clearly holds. Suppose the lemma holds for $n-1$ and let $h_1,\ldots, h_n$ satisfy the assumptions in the lemma statement. Let $\alpha_1,\ldots, \alpha_n$ satisfy
\begin{eqnarray}
	\sum_{i=1}^n h_i^{\otimes n-1} \alpha_i = 0. \label{lisum}
\end{eqnarray}
To finish the proof we will show that $\alpha_1$ must be zero which can be generalized to any $\alpha_i$ without loss of generality.
Let $H_1$ and $H_2$ be Hilbert spaces and let $\hs\left( H_1, H_2 \right)$ be the space of Hilbert-Schmidt operators from $H_1$ to $H_2$. Hilbert-Schmidt operators are a closed subspace of bounded linear operators. Proposition 2.6.9 in \cite{kadison83} states that for a pair of Hilbert spaces $H_1, H_2$ there exists an unitary operator $U:H_1 \otimes H_2 \to \hs\left( H_1,H_2 \right)$ such that $U(g_1\otimes g_2) = g_1 \l<g_2, \cdot\r>$. Applying this operator to (\ref{lisum}) we get
\begin{eqnarray}
	\sum_{i=1}^n h_i^{\otimes n-2}\l<h_i, \cdot \r> \alpha_i = 0. \label{lioper}
\end{eqnarray}
Because $h_1$ and $h_n$ are linearly independent we can choose $z$ such that $\l<h_1,z\r> \neq 0$ and $z\perp h_n$. Plugging $z$ into (\ref{lioper}) yields
\begin{eqnarray*}
	\sum_{i=1}^{n-1} h_i^{\otimes n-2}\l<h_i, z \r> \alpha_i = 0 
\end{eqnarray*}
and therefore $\alpha_1=0$ by the inductive hypothesis.
\end{proof}

\section{Proofs of Theorems}
With the tools developed in the previous sections we can now prove our theorems. First we introduce one additional piece of notation. For a function $p$ on a domain $\sX$ we define $p^{\times k}$ as simply the product of the function $k$ times on the domain $\sX^{\times k}$, $\underbrace{p(\cdot)\cdots p(\cdot)}_{\text{k times}}$. For a measure the notation continues to denote the standard product measure.

Finally will need the following technical lemma to connect the product of Radon-Nikodym derivatives to product measures. The proof is straightforward and can be found in the appendix.
\begin{lemma} \label{lem:radprod}
	Let $\left( \Psi, \sG \right)$ be a measurable space, $\eta$ and $\gamma$ a pair of bounded measures on that space, and $f$ a nonnegative function in $L^1\left( \gamma \right)$ such that, for all $A \in \sG$, $\eta\left( A \right)=\int_A f d\gamma$. Then for all $n$, for all $B \in \sG^{\times n}$ we have
\begin{eqnarray*}
	\eta^{\times n}\left( B \right) = \int_B f^{\times n} d\gamma^{\times n}.
\end{eqnarray*}
\end{lemma}
\begin{proof}{\bf of Theorem \ref{thm:ident}}
	We will proceed by contradiction. Suppose there exist two different mixtures of measures $\sP = \sum_{i=1}^l \delta_{\mu_i} a_i \neq \sP' = \sum_{j=1}^m \delta_{\nu_j}b_j$, such that
	\begin{eqnarray*} \label{grr}
		\sum_{i=1}^{l} a_i {\mu_i}^{\times 2m-1} = \sum_{j=1}^{m} b_j {\nu}_j^{\times 2m-1}
	\end{eqnarray*}
and $l\le m$. From our assumption on representation we know $\mu_i \neq \mu_j$ for all $i\neq j$ and similarly for $\nu_1,\ldots, \nu_m$. We will also assume that $\mu_i \neq \nu_j$ for all $i,j$. Were this not true we could simply subtract the smaller of the common terms from both sides of (\ref{grr}) and normalize to yield another pair of distinct mixtures of measures with fewer components and no shared terms, $\sQ$ and $\sQ'$. Let $\sQ$ have $m'$ components and $\sQ'$ have $l'$ with $m'\ge l'$. If $m\neq m'$ then we can apply Lemma \ref{lem:noident} to give us $V_{2m'-1}\left( \sQ \right)= V_{2m'-1}\left( \sQ' \right)$ and proceed as usual.

	Let $\xi = \sum_{i=1}^l \mu_i + \sum_{j=1}^m \nu_j$. Clearly $\xi$ dominates $\mu_i$ and $\nu_j$ for all $i,j$ so we can define Radon-Nikodym derivatives $p_i = \frac{d \mu_i}{d \xi}$, $q_j = \frac{d \nu_j}{d \xi}$ which are in $L^1\left( \Omega , \sF, \xi \right)$. We can assert that these derivatives are everywhere nonnegative without issue. Clearly no two of these derivatives are equal. If one of the derivatives were a scalar multiple of another, for example $p_1 = \alpha p_2$ for some $\alpha \neq 1$, it would imply
\begin{eqnarray*}
	\mu_1\left( \Omega \right) = \int_{\Omega} p_1 d\xi = \int \alpha p_2 d\xi=\alpha.
\end{eqnarray*}
This is not true so no pair of these derivatives are collinear.

Lemma \ref{lem:radprod} tells us that, for any $R \in \sF^{\times 2m-1}$  we have
\begin{eqnarray*}
	\int_R \sum_{i=1}^{l} a_i p_i^{\times 2m-1} d\xi^{\times 2m-1} 
	&=&  \sum_{i=1}^{l} a_i \mu_i^{\times 2m-1}\left( R \right)\\
	&=&  \sum_{j=1}^{m} b_j \nu_j^{\times 2m-1}\left( R \right)\\
	&=& \int_R \sum_{j=1}^{m} b_j q_j^{\times 2m-1}d\xi^{\times 2m-1}.
\end{eqnarray*}
Therefore
\begin{eqnarray} \label{eqn:radonequal}
	\sum_{i=1}^{l} a_i p_i^{\times 2m-1} = \sum_{j=1}^{m} b_j q_j^{\times 2m-1}
\end{eqnarray}
$\xi^{\times 2m-1}$-almost everywhere (Proposition 2.23 in \cite{folland99}).
We will now show for all $i,j$ that $p_i \in L^2 \left( \Omega, \sF, \xi \right)$ and $q_j \in L^2\left( \Omega, \sF, \xi \right)$. We will argue this for $p_1$ which will clearly generalize to the other elements. First we will show that $p_1 \le 1$ $\xi$-almost everywhere. Suppose this were not true and that there exists $A \in \sF$ with $\xi\left( A \right)>0$ and $p_1\left( A \right)>1$. Now we would have 
\begin{eqnarray*}
	\mu_1\left( A \right)
	=\int_A p_1 d\xi
	> \int_A 1 d\xi
	= \xi\left( A \right)
	=\sum_{i=1}^l\mu_i\left( A \right) + \sum_{j=1}^m \nu_j\left( A \right)
	\ge \mu_1\left( A \right)
\end{eqnarray*}
a contradiction. Evaluating directly we get
\begin{eqnarray*}
	\int p_1(\omega)^2 d \xi\left( \omega \right)
	&\le& \int 1 d \xi\left( \omega \right)\\
	&=& \xi\left( \Omega \right)\\
	&=& l+m,
\end{eqnarray*}
so $p_1 \in L^2\left( \Omega, \sF, \xi \right)$.
Applying the $U^{-1}$ operator from Lemma \ref{lem:l2prod} to (\ref{eqn:radonequal}) yields
\begin{eqnarray*}
		\sum_{i=1}^{l} a_i p_1^{\otimes 2m-1} = \sum_{j=1}^{m} b_j q_j^{\otimes 2m-1}.
\end{eqnarray*}
Since $l+m \le2m$ Lemma \ref{lem:linind} states that $p_1^{\otimes 2m-1},\ldots,p_{l}^{\otimes 2m-1},q_1^{\otimes 2m-1},\ldots,q_{m}^{\otimes 2m-1}$ are all linearly independent and thus $a_i = 0$ and $b_j = 0$ for all $i,j$, a contradiction.
\end{proof}

\begin{proof}{\bf of Theorem \ref{thm:noident}}
To prove this theorem we will construct a pair of different mixture of measures, $\sP \neq \sP'$ which both contain $m$ components and satisfy $V_{2m-2}\left( \sP \right) = V_{2m-2}\left( \sP' \right)$.

From our definition of $\left( \Omega, \sF \right)$ we know there exists $F\in \sF$ such that $F, F^C$ are nonempty. Let $f\in F$ and $f' \in F^C$. It follows that $\delta_{f} \neq \delta_{f'}$ are different probability measures on $\left( \Omega, \sF \right)$. Because $\delta_{f}$ and $\delta_{f'}$ are dominated by $\xi = \delta_{f} + \delta_{f'}$ we know that there exists a pair of measurable functions $p, p'$ such that, for all $A$, $\delta_{f}\left( A \right) = \int_A p d\xi$ and $\delta_{f'}\left( A \right) = \int_{A} p'd\xi$. We can assert that $p$ and $p'$ are nonnegative without issue.

 From the same argument we used in the proof of Theorem \ref{thm:ident} we know $p,p' \in L^2\left( \Omega, \sF, \xi \right)$. Let $H_2$ be the Hilbert space generated from the span of $p,p'$. Let $(\varepsilon_i)_{i=1}^{2m}$ be $2m$ distinct elements of $\left[ 0,1 \right]$ and let $\left( p_i \right)_{i=1}^{2m}$ be elements of $L^1(\Omega, \sF, \xi)$ with $p_i = \varepsilon_i p + \left( 1-\varepsilon_i \right)p'$. Clearly $p_i$ is a pdf over $\xi$ for all $i$ and there are no pairs in this collection which are collinear. Let $H_2$ be the Hilbert space generated from the span of $p$ and $p'$. Since $H_2$ is isomorphic to $\rn^2$ there exists a unitary operator $U:H_2 \to \rn^2$. From Lemma \ref{lem:unitprod} there exists a unitary operator $U_{2m-2}:H_2^{\otimes 2m-2} \to {\rn^2}^{\otimes 2m-2}$ with $U_{2m-2}\left( h_1 \otimes\cdots \otimes h_{2m-2} \right) = U(h_1) \otimes \cdots \otimes U(h_{2m-2})$. Because $U$ is unitary the set $U_{2m-2}\left( \span\left( \left\{ h^{\otimes 2m-2}: h \in H_2 \right\} \right) \right)$ maps exactly to the set $\span\left( x^{\otimes 2m-2}:x \in \rn^2 \right)$.  An order $r$ tensor, $A_{i_1,\ldots,i_r}$, is {\it symmetric} if $A_{\psi\left( i_1 \right),\ldots,\psi\left( i_r \right)} = A_{i_1,\ldots,i_r}$for any $i_1,\ldots, i_r$ and permutation $\psi$. A consequence of Lemma 4.2 in \cite{symtensorrank} is that $\span\left( \l\{x^{\otimes 2m-2}:x \in \rn^2\r\} \right)\subset S^{2m-2}(\cn^2)$ is exactly the space of all symmetric order $2m-2$ tensors over $\cn^2$.

 From Proposition 3.4 in \cite{symtensorrank} it follows that the dimension of $S^{ 2m-2 }\left( \cn^2 \right)$ is $\left( \begin{array}{c} 2+ 2m-2-1 \\ 2m-2 \end{array} \right) = 2m-1$. From this we get that $\dim\left( \span\left( \left\{ h^{\otimes 2m-2}: h \in H_2 \right\} \right)\right)\le2m-1$.

The bound on the dimension of $\span\left( \left\{ h^{\otimes 2m-2}: h \in H_2 \right\} \right)$ implies that $\left( p_i^{\otimes 2m-2} \right)_{i=1}^{2m}$ are linearly dependent. Conversely Lemma \ref{lem:linind} implies that removing a single vector from $\left( p_i^{\otimes 2m-2} \right)_{i=1}^{2m}$ yields a set of vectors which are linearly independent. It follows that there exists $\left( \alpha_i \right)_{i=1}^{2m}$ with $\alpha_i \neq 0$ for all $i$ and

\begin{eqnarray*}
	\sum_{i=1}^{2m}\alpha_i p_i^{\otimes 2m-2} = 0.
\end{eqnarray*}

Without loss of generality we will assume that $\alpha_i<0$ for $i\in \left[ k \right]$ with $k\le m$. From this we have 

\begin{eqnarray}\label{foo}
	\sum_{i=1}^{k}-\alpha_i p_i^{\otimes 2m-2}=\sum_{j=k+1}^{2m}\alpha_j p_j^{\otimes 2m-2}.
\end{eqnarray}
From Lemma \ref{lem:l2prod} we have
\begin{eqnarray*}
	\sum_{i=1}^{k}-\alpha_i p_i^{\times 2m-2}=\sum_{j=k+1}^{2m}\alpha_j p_j^{\times 2m-2}
\end{eqnarray*}
and thus
\begin{eqnarray*}
	\int \sum_{i=1}^{k}-\alpha_i p_i^{\times 2m-2} d\xi^{\times 2m-2 }&=&\int \sum_{j=k+1}^{2m}\alpha_j p_j^{\times 2m-2}d\xi^{\times 2m-2 }\\
	 \Rightarrow \sum_{i=1}^{k}-\alpha_i &=&\sum_{j=k+1}^{2m}\alpha_j.
\end{eqnarray*}
Let $r=\sum_{i=1}^{k}-\alpha_i$. We know $r >0$ so dividing both sides of (\ref{foo}) by $r$ gives us
\begin{eqnarray*}
	\sum_{i=1}^{k}-\frac{\alpha_i}{r} p_i^{\otimes 2m-2}=\sum_{j=k+1}^{2m}\frac{\alpha_j}{r} p_j^{\otimes 2m-2}
\end{eqnarray*}
and the left and the right side are convex combinations. Let $\left( \beta_i \right)_{i=1}^{2m}$ positive numbers with $\beta_i = \frac{-\alpha_i}{r}$ for $i \in \left\{ 1,\ldots,k \right\}$ and $\beta_j = \frac{\alpha_j}{r}$ for $j\in \left\{ k+1,\ldots,2m \right\}$. This gives us
\begin{eqnarray*}
	\sum_{i=1}^{k}\beta_i p_i^{\otimes 2m-2}=\sum_{j=k+1}^{2m}\beta_j p_j^{\otimes 2m-2}.
\end{eqnarray*}
 It follows that
\begin{eqnarray*}
	\sum_{i=1}^{k} \beta_i p_i^{\otimes m-1}\otimes p_i^{\otimes m-1}&=&\sum_{j=k+1}^{2m}\beta_j p_j^{\otimes m-1}\otimes p_i^{\otimes m-1}.
\end{eqnarray*}
We will now show that $k=m$. Suppose $k<m$. Then $p_1^{\otimes m-1}, \ldots , p_{k+1}^{\otimes m-1}$ are linearly independent. From this we know that there exists $z$ such that $z\perp p_i^{\otimes m-1}$ for $i\in [k]$ but $z$ is not orthogonal to $p_{k+1}^{\otimes m-1}$. Using this vector we have
\begin{eqnarray*}
	\l<\sum_{i=1}^{k}\beta_i p_i^{\otimes 2m-1},z\otimes z\r>
	&=&\sum_{i=1}^{k}\beta_i\l<z, p_i^{\otimes m-1}\r>\l<z, p_i^{\otimes m-1}\r>\\
	&=&0 
\end{eqnarray*}
but
\begin{eqnarray*}
	\l< \sum_{i=k+1}^{2m}\beta_i p_i^{\otimes m-1}\otimes p_i^{\otimes m-1}, z\otimes z\r>
	&=& \sum_{i=k+1}^{2m}\beta_i\l<  p_i^{\otimes m-1}, z\r>\l<  p_i^{\otimes m-1}, z\r>\\
	&>&0
\end{eqnarray*}
and thus $k=m$.

Now we have
\begin{eqnarray*}
	\sum_{i=1}^m \beta_i p_i^{\otimes 2m-2} = \sum_{j=m+1}^{2m}\beta_j p_j^{\otimes 2m-2}.
\end{eqnarray*}
Applying Lemma \ref{lem:l2prod} we get that
\begin{eqnarray*}
	\sum_{i=1}^m \beta_i p_i^{\times 2m-2} = \sum_{j=m+1}^{2m}\beta_j p_j^{\times 2m-2}.
\end{eqnarray*}
From Lemma \ref{lem:radprod} we have,
\begin{eqnarray*}
	 \sum_{i=1}^m  \beta_i \left( \varepsilon_i \delta_f + \left( 1-\varepsilon_i \right) \delta_{f'} \right)^{\times 2m-2}   &=& \sum_{j=m+1}^{2m} \beta_j \left( \varepsilon_j \delta_f + \left( 1-\varepsilon_j \right) \delta_{f'} \right)^{\times 2m-2}.
\end{eqnarray*}
Setting $\mu_i = \left( \varepsilon_i \delta_{f} + \left( 1-\varepsilon_i \right)\delta_{f'} \right)$ yields
\begin{eqnarray*}
	\sum_{i=1}^m  \beta_i \mu_i^{\times 2m-2}  &=& \sum_{j=m+1}^{2m} \beta_j \mu_j^{\times 2m-2}.
\end{eqnarray*}
Thus setting $\sP = \sum_{i=1}^m  \beta_i \delta_{\mu_i}$ and $\sP' = \sum_{j=m+1}^{2m} \beta_j \delta_{\mu_j}$ gives us $V_{2m-2}\left( \sP \right) = V_{2m-2}\left( \sP' \right)$ and $\sP \neq \sP'$ by construction. 
\end{proof}
\subsection{Discussion of the Proof of Theorem \ref{thm:noident}}
In the previous proof we could have replaced $\delta_f,\delta_{f'}$ with any distinct pair of probability measures on $\left( \Omega, \sF \right)$. Thus the pair $\sP, \sP'$ are not pathological because of some property of each individual mixture component, but because of geometry of the mixture components considered as a whole. The measures $\mu_1,\ldots,\mu_{2n}$ are a convex combinations of $\delta_f$ and $\delta_{f'}$ and therefore lie in a one dimensional affine subspace of $\dd$. The space of Bernoulli measures similarly lie in a subspace between two measures, the point mass at $0$ and the point mass at $1$. Given a mixture of Bernoulli distributions, the sum of iid samples of Bernoulli random variables is a binomial distribution. We can draw a connection between our result and the identifiability of mixtures of binomial distributions.

Consider $\sP$ as mixture of $m$ Bernoulli distributions with parameters $\lambda_1,\ldots,\lambda_m$ and weights $w_1,\ldots w_m$. Suppose we have $n$ samples in each random group. If we let $Y_i$ be the sum of the random group $\bX_i$ then the probability law of $Y_i$ is a mixture of binomial random variables. Let $p(\lambda,n)$ be the distribution of a Bernoulli random variable with parameters $n$ and $\lambda$. Specifically we have that the distribution of $Y_i = \sum_{i=1}^m w_i p(\lambda_i,n)$. In \cite{blischke64} it was shown that $n \ge 2m-1$ is a necessary and sufficient condition for the identifiability of the parameters $\lambda_1,\ldots, \lambda_m$ from the samples $Y_i$. We find these similarities provoking but are not prepared to make more precise connections at this time.
\section{Conclusion}
In this paper we have proven a fundamental bound on the identifiability of mixture models in a nonparametric setting. Any mixture with $m$ components is identifiable with groups of samples containing $2m-1$ samples from the same latent probability measure. We show that this bound is tight by constructing a mixture of $m$ probability measures which is not identifiable with groups of samples containing $2m-2$. These results hold for any mixture over any domain with at least two elements.
\bibliography{rvdm}
\appendix
\section{Additional Proofs}
\begin{proof}{\bf of Lemma \ref{lem:l2prod}}
 Example 2.6.11 in \cite{kadison83} states that for any two $\sigma$-finite measure spaces $\left( S,\mathscr{S}, m \right), \left( S',\mathscr{S}', m' \right)$ there exists a unitary operator $U: L^2\left( S,\mathscr{S}, m \right) \otimes L^2 \left( S',\mathscr{S'}, m' \right) \to L^2\left( S\times S', \mathscr{S}\times \mathscr{S'}, m\times m' \right)$ such that, for all $f,g$,
\begin{eqnarray*}
	U(f\otimes g) = f(\cdot)g(\cdot).
\end{eqnarray*}
Because $\left( \Psi, \sG, \eta \right)$ is a $\sigma$-finite measure space it follows that $\left( \Psi^{\times m}, \sG^{\times m}, \eta^{\times m} \right)$ is a $\sigma$-finite measure space for all $m\in \mathbb{N}$. We will now proceed by induction. Clearly the lemma holds for $n=1$. Suppose the lemma holds for $n-1$. From the induction hypothesis we know that there exists a unitary transform $U_{n-1}: L^2\left( \Psi, \sG, \eta \right)^{\otimes n-1} \to L^2 \left( \Psi^{\times n-1} ,\sG ^{\times n-1}  , \eta^{n-1} \right)$ such that for all simple tensors$ f_1\otimes\cdots \otimes f_{n-1} \mapsto f_1(\cdot)\cdots f_{n-1}\left( \cdot \right)$. Combining $U_{n-1}$ with the identity map via Lemma \ref{lem:unitprod} we can construct a unitary operator $T_n: L^2\left( \Psi, \sG, \eta \right)^{\otimes n-1} \otimes L^2\left( \Psi, \sG, \eta \right) \to L^2 \left( \Psi^{\times n-1} ,\sG ^{\times n-1}  , \eta^{n-1} \right) \otimes L^2\left( \Psi, \sG, \eta \right)$, which maps $f_1\otimes\cdots\otimes f_{n-1}  \otimes f_n \mapsto f_1(\cdot)\cdots f_{n-1}(\cdot) \otimes f_n$

 From the aforementioned example there exists a unitary transform $K_n:L^2\left( \Psi^{n-1},\sG^{\times n-1}, \eta^{n-1} \right)\otimes L^2\left( \Psi,\sG, \eta \right)\to L^2 \left( \Psi^{\times n-1} \times \Psi,\sG ^{\times n-1} \times \sG , \eta^{n-1}\times \eta\right)$ which maps $f\otimes f' \mapsto f\left( \cdot \right)f'\left( \cdot \right)$. Defining $U_n(\cdot)= K_n\left( T_n \left( \cdot \right) \right)$ yields our desired unitary transform.
\end{proof}

\begin{proof}{\bf of Lemma \ref{lem:unitprod}}
	Proposition 2.6.12 in \cite{kadison83} states that there exists a continuous linear operator $\tilde{U}:H_1 \otimes \cdots \otimes H_n \to H_1' \otimes \cdots \otimes H_n'$ such that $\tilde{U}\left( h_1 \otimes\cdots \otimes h_n \right) = U_1(h_1) \otimes \cdots \otimes U_n(h_n)$ for all $h_1 \in H_1 ,\cdots, h_n \in H_n$. Let $\widehat{H}$ be the set of simple tensors in $H_1 \otimes \cdots \otimes H_n$ and $\widehat{H}'$ be the set of simple tensors in $H_1'\otimes \cdots \otimes H_n'$. Because $U_i$ is surjective for all $i$, clearly $\tilde{U}(\widehat{H}) = \widehat{H}'$. The linearity of $\tilde{U}$ implies that $\tilde{U}(\span(\widehat{H}))= \span(\widehat{H}')$. Because $\span(\widehat{H}')$ is dense in $H_1'\otimes \cdots \otimes H_n'$ the continuity of $\tilde{U}$ implies that $\tilde{U}(H_1\otimes\cdots \otimes H_n) = H_1'\otimes \cdots \otimes H_n'$ so $\tilde{U}$ is surjective. All that remains to be shown is that $\tilde{U}$ preserves the inner product. By the continuity of inner product we need only show that $\l<h, g\r>=\l<\tilde{U}(h), \tilde{U}(g)\r>$ for $h,g \in \span(\widehat{H})$. With this in mind let $h_1,\ldots, h_N,g_1,\ldots,g_M \in \widehat{H}$. We have the following
\begin{eqnarray*}
	\l<\tilde{U}\l(\sum_{i=1}^N h_i\r),\tilde{U}\l(\sum_{j=1}^M g_j\r) \r>
	&=& \l<\sum_{i=1}^N \tilde{U}\l(h_i\r),\sum_{j=1}^M \tilde{U}\l(g_j\r) \r>\\
	&=& \sum_{i=1}^N\sum_{j=1}^M\l< \tilde{U}\l(h_i\r), \tilde{U}\l(g_j\r) \r>\\
	&=& \sum_{i=1}^N\sum_{j=1}^M\l< h_i, g_j \r>\\
	&=& \l< \sum_{i=1}^Nh_i, \sum_{j=1}^M g_j \r>.
\end{eqnarray*}
We have now shown that $\tilde{U}$ is unitary which completes our proof.
\end{proof}

\begin{proof}{\bf of Lemma \ref{lem:radprod}}
	The fact that $f$ is positive and integrable implies that the map $S \mapsto \int_S f^{\times n}d\gamma^{\times n}$ is a bounded measure on $\l(\Psi^{\times n}, \sG^{\times n}\r)$ (see \cite{folland99} Exercise 2.12). 

Let $R= R_1 \times\ldots\times R_n$ be a rectangle in $\sG^{\times n}$. Let $\mathds{1}_S$ be the indicator function for a set $S$. Integrating over $R$ and using Tonelli's theorem we get
\begin{eqnarray*}
	\int_R f^{\times n} d \gamma^{\times n}
	&=& \int \mathds{1}_Rf^{\times n}d \gamma^{\times n}\\
	&=& \int \mathds{1}_Rf^{\times n}d \gamma^{\times n}\\
	&=& \int \l(\prod_{i=1}^n \mathds{1}_{R_i}(x_i)\r)\l(\prod_{j=1}^n f(x_j)\r)d \gamma^{\times n}\left( x_1,\ldots,x_n \right)\\
	&=& \int\cdots\int \l(\prod_{i=1}^n \mathds{1}_{R_i}(x_i)\r)\l(\prod_{j=1}^n f(x_j)\r)d \gamma(x_1)\cdots d\gamma(x_n)\\
	&=& \int\cdots\int \l(\prod_{i=1}^n \mathds{1}_{R_i}(x_i) f(x_i)\r)d \gamma(x_1)\cdots d\gamma(x_n)\\
	&=&  \prod_{i=1}^n\l(\int \mathds{1}_{R_i}(x_i) f(x_i)d \gamma(x_i)\r)\\
	&=&  \prod_{i=1}^n\eta(R_i)\\
	&=&  \eta^{\times n}(R).
\end{eqnarray*}
Any product probability measure is uniquely determined by its measure over the rectangles (this is a consequence of Lemma 1.17 in \cite{fomp} and the definition of product $\sigma$-algebra) therefore, for all $B\in \sG^{n}$,
\begin{eqnarray*}
	\eta^{\times n}\left( B \right) = \int_B f^{\times n} d\gamma^{\times n}.
\end{eqnarray*}

\end{proof}
\end{document}